\documentclass[graybox]{svmult}

\usepackage{mathptmx}       
\usepackage{helvet}         
\usepackage{courier}        
\usepackage{type1cm}        
\usepackage{float}  
\usepackage{adjustbox}
\usepackage{adjustbox}
\usepackage{graphicx}
\usepackage{bbm}
\usepackage{amssymb}
\usepackage{amsmath}
\usepackage{makecell}
\usepackage{booktabs}
\usepackage{amsfonts}
\usepackage{makeidx}         
\usepackage{graphicx}        
\usepackage{multicol}        
\usepackage[bottom]{footmisc}

\makeindex             

\begin{document}

\title*{Probabilistic Random Indexing for Continuous Event Detection}
\author{Yashank Singh and Niladri Chatterjee}
\institute{Yashank Singh, \email{syash0516@gmail.com}
\\ Niladri Chatterjee, \email{niladri@maths.iitd.ac.in}}
%
%
\maketitle

\small
\abstract{The present paper explores a novel variant of Random Indexing (RI) based representations for encoding language data with a view to using them in a dynamic scenario where events are happening in a continuous fashion. As the size of the representations in the general method of one-hot encoding grows linearly with the size of the vocabulary, they become non-scalable for online purposes with high volumes of dynamic data. On the other hand, existing pre-trained embedding models are not suitable for detecting happenings of new events due to the dynamic nature of the text data. The present work  addresses this issue by using a novel RI representation  by imposing a probability distribution on the number of randomized entries which leads to a class of RI representations. It also provides a rigorous  analysis of the goodness of the representation methods to encode semantic information in terms of the probability of orthogonality. Building on these ideas we propose an algorithm that is log-linear with  the size of vocabulary to track the semantic relationship of a query word to other words for suggesting the events that are relevant to the word in question. We ran simulations using the proposed algorithm for tweet data specific to three different events and present our findings. The proposed probabilistic RI representations are found to be much faster and scalable than Bag of Words (BoW) embeddings while maintaining accuracy in depicting semantic relationships.}
\vspace{-0.3cm}
\section{Introduction}
\label{sec:1}
Event detection has been a common application of Machine Learning and NLP. The classic tech-
niques in this regard typically involve Word Co-occurrence Matrix and its Singular Value Decomposition (SVD) \cite{SVD} to track meaning
and relationship between different words to detect an event. However, in modern times these
methods are found to be slow to cope up with the huge amount of new information that is created everyday on internet platforms, twitter being one of the most popular of them in this regard. Applications of these methods for detection
of dynamic events suffer from three major problems, namely arbitrariness of the underlying
language (English, here), exponentially increasing volume of data, and the curse of dimensionality.
\par Computer processing of natural languages is a generally perceived as a huge challenge because
\begin{itemize}
    \item words are often ambiguous, i.e. one word can have several meanings (polysemy)
    \item several words
may refer to the same concept (synonymy).
\end{itemize} On the other hand, in the context of online learning,
documents may arrive continuously. Many of them may have some unseen words and thereby increasing the size of the database, and also expanding the size of the vocabulary. The last one leads to the other problem i.e. the curse of dimensionality. In one-hot encoding documents are represented with a vector of a dimensionality equal to size of
the vocabulary which in a real environment is equal to the total number of words found so far.

An additional problem emerges with respect to event detection if the information is varying with time. Even if one divides the corpora on the basis of time, and applies SVD on the obtained co-occurrence matrix to track word relationships, the resulting reduced matrices would not be
comparable in general, because of their dimensional variations. This, in turn, would give little or no idea
about shift in a word’s association with other words. 
\par Some algorithms have been reported in the
literature with the aim of alleviating one or more of the above-mentioned issues. Among them are
 Latent Semantic Indexing (LSI) \cite{lsi}, Probabilistic Latent Semantic
Indexing (PLSI) \cite{PLSI}, Latent Dirichlet Allocation (LDA), \cite{blei2003latent} and other recent embedding-based schemes, such as Word2Vec \cite{word2vec}, BERT \cite{Devlin2019BERTPO}. However, these methods are computationally expensive, or require access to the whole
term-document frequency matrix during the semantic space construction. Although pre-trained embeddings are generally available for them, those are not capable of capturing the temporal variations which is essential for detecting new events. These disadvantages limit
their incorporation into an online environment. 

The present paper addresses the above problems by introducing the concept of Probabilistic Random Indexing of context vectors of words. We show that the Probabilistic Random Indexing space can substitute the conventional method of one-hot encoding and other learned embedding based approaches for representing relationship between words. We further show how the probability of orthogonality can affect accuracy of these words representations significantly to convey meaningful information. We also comment on the choice of dimension of random vector “n” and the number of non-zero entries “r” to be chosen for a data set and develop an algorithm to suggest a list of words which
could be related to an event.

The paper is organised as follows. Section 2 provides a brief description of several works where
Random Indexing has been used for some NLP applications. Section 3 provides a description of the Random
Indexing technique in detail. Section 4 discusses the novel probability distribution conceptualised in the work.
Section 5 describes event detection scheme. In Section 6 the experimental set up is explained, and the results
obtained are presented. Section 7 concludes the paper with some future research goal.

\vspace{-0.3cm}
\section{Related Work}
\label{subsec : 2.1}
Random Indexing (RI) based approaches for various NLP applications  have been pursued over the last two decades.  For illustration, Bruzon et al. \cite{Bruzn2015ExploringRI} exploits the RI method in the realm of profile learning. Kanvera et al. \cite{Kanerva2000RandomIO} leverage Random Indexing for Latent Semantic Analysis of text samples based on co-occurrences. Cohen et al. \cite{cohen2010reflective} proposed Reflective RI to alleviate shortcomings of Random Indexing in determining indirect relations between words. Basil et al. \cite{basile2015temporal} introduced the concept of creating temporally local word-spaces using Random Indexing for linguistic analysis of word relationships.  In another work \cite{Basile2016TemporalRI} the same authors use a similar approach by creating geometrical spaces of word meanings that consider several periods of time to analyse word relationships in news data and analysis of their cultural, social, and linguistic phenomenon.

\par Random Indexing has also been used for language detection by Joshi et al. \cite{Joshi2014LanguageRU}.  Sandin et al. \cite{Sandin2016RandomIO} extend the idea of Random Indexing to multidimensional arrays for inference of higher order statistical relationships in data in an online setting.  Chatterjee and Mohan \cite{chatterjee2007extraction} employed Random Indexing for extractive text summarization by removing the redundancy from text documents. Chatterjee and Sahoo \cite{Chatterjee2013EffectON, chatterjee2015random} exploit the Random Indexing of context vectors for the task of extractive text summarization. However, in their approach they create the word space by distributing +1’s in the upper half of the random vector and -1’s in the lower half sparsely. Some work has also been done exploiting the shift in the semantic vectors of words in corpus for the purpose of event detection  \cite{jurgens2009event}. This work however does not propose any distribution on the number of non-zero entries. Therefore, they cannot capture the dynamics of an online scenario.
\par The novelty of the present work is that it does not impose any restriction on placement of non-zero elements. Here, the +1’s and -1’s are distributed randomly throughout the context vector effectively reducing the length of the context vector. Furthermore, the novel idea of introducing a probability distribution on the number of non-zero entries has not been exploited in any previous work, as the aforementioned works fix the number of non-zero entries to a deterministic value. Moreover, no analysis with respect to probability of orthogonality has been done in the aforementioned works. The present work studies the effect of imposing a probability distribution on the number of non-zero entries in  the proposed representation scheme. Present work further provides comparison of the proposed representation method with the one-hot encoding (BoW) as well as the scheme presented in \cite{Chatterjee2013EffectON}.

 \section{Random Indexing}
\label{sec:2}

Random Indexing introduced by \cite{Kanerva2000RandomIO} for Latent Semantic Analysis (LSA) and is primarily pillared on three main assumptions:
\begin{enumerate}
\item Words having similar meaning occur within similar contexts, i.e. distributional hypothesis as explained in \cite{10.1145/365628.365657}
    \item Projection of a high dimensional space onto a lower dimensional space can be achieved such that the relative distances are not affected \cite{Johnson1984ExtensionsOL}
    \item Existence of a large amount of pseudo-orthogonal directions in high dimensional space \cite{HechtNielsen94ContextVectors}
\end{enumerate}

The approach proposed in \cite{Kanerva2000RandomIO} was further developed in \cite{Sahlgren2005AnIT} that formalized Random Indexing as a two-step process:  
\begin{itemize}
    \item First, each word in the corpus is assigned a context vector as in Eqn. (\ref{context}) which is a high dimensional vector sparsely filled with +1's and -1's
    \item In the second step, semantic vectors are constructed from the context vectors for each word w by scanning the text, and averaging over the context vectors of words occurring within a sliding context window each time the word w occurs in the text document.
\end{itemize}
In the definitions below and throughout this paper \textit{n} denotes the dimension of context vector and \textit{r} denotes the number of non-zero entries. We define RI space to be the vector space spanned by the context vectors.

\textbf{Context Vectors} : These are the vectors associated with each word specifying its context value in a n-dimensional domain. In the one hot encoding case, these vectors can be defined on N-dimensional space, where N is the total number of distinct words in vocabulary, as follows for a word $w_i$ in the vocabulary:
\begin{equation}
c_{i} = (0,0,0, \dots, 1_{i^{th}},...0,0,0) 
\end{equation}
In case of Probabilistic RI representations with $r=5$ non-zero entries where each entry can take the value either +1 or -1, a n-dimensional context vector for word $w_i$ may look like:
\begin{equation}
c_{i} = (0,+1,0,\dots,-1,0,+1,\dots, -1,0,+1,0) 
\label{context}
\end{equation}
Here, the indexes of the non-zero entries are selected at random in the n-dimensional vector, and each non-zero entry is selected to be +1 or -1 with probability $0.5$.
\par \textbf{Semantic Vectors} : These are the vectors that are part of the word co-occurrence matrix. The $i^{th}$ row of the matrix corresponds to the semantic vector of the word $w_i$ in the corpora which can be calculated by adding the context vectors of all the words in context range of the occurrences of word i as follows:
\begin{equation}
sv_{i} = \sum_{d\in W} \sum_{-m<i<m}c_i   
\label{semantic}
\end{equation}

Here, $W$ denotes the vocabulary and d is the occurrence of word $w_i$ in data set.\\

The advantage of constructing Random Indexing (RI) based vectors instead of one-hot vectors is the following. It allows to project the higher dimensional word space to a lower dimensional space spanned by randomly indexed context vectors which are nearly orthogonal, unlike the one-hot vectors which are orthogonal to each other. This is achieved by assigning non-zero entries from a set $S$ to some small, say $r$, but randomly chosen coordinates of vectors of pre-defined size $k$. Typically, the set $S$ is $\{+1, -1\}$, and $r$ is an even number. Note that if $r = 1$ then it boils down to one-hot encoding where the vectors are orthogonal.
\par The advantage of compromising on total orthogonality is that here a lower dimensional vector can represent more words (as $r$ increases) to represent the data set and track relationship between different words. The present work establishes a relation between the probability of orthogonality and accuracy of the representation to convey meaningful information. Moreover, as the size of semantic vector remains same through time, it becomes meaningful to compare semantic vectors of a word to track any possible shift in meaning of the word with time, This allows one to add a temporal component to Random Indexing, which we term as Temporal Random Indexing. The word capacity of the random indexed space with $k$ as the size of context vector and $r$ as number of non-zero entries is shown in Fig 1a. It grows exponentially as compared to linearly in simple case.
\vspace{-0.3cm}
\begin{subfigures}
\begin{figure}[!h]
   \begin{minipage}{0.48\textwidth}
     \centering
     \includegraphics[width=\linewidth, height = 4cm ]{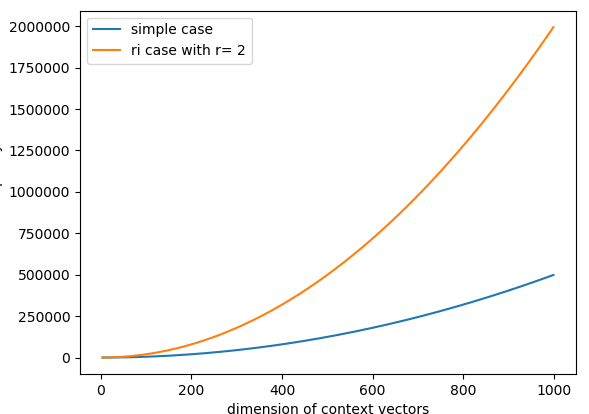}
     \caption{Word Capacity of Standard RI Case when $r=1:$ $^{n}C_{r}$ vs. RI case : $^{n}C_r\times 2^r r=2$ }
    \label{fig1a}
   \end{minipage}\hfill
   \begin{minipage}{0.48\textwidth}
     \centering
     \includegraphics[width=\linewidth, height=4cm]{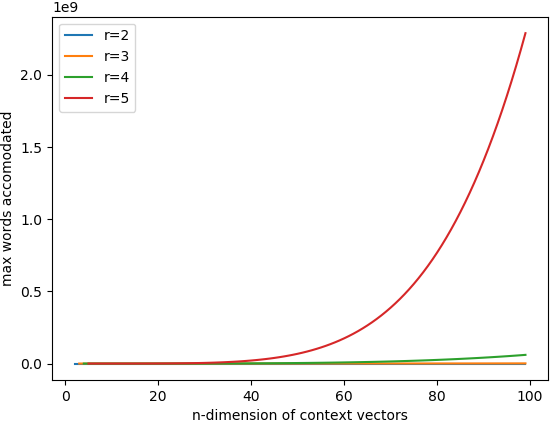}
     \caption{Word Capacity of Standard RI case with choice between 1 and -1 for each entry : $^{n}C_{r}*2^{r}$ }
    \label{fig1b}
   \end{minipage}
\end{figure}
\end{subfigures}
\vspace{-0.3cm}
\subsection{Word Representation Capacity }
\vspace{-0.3cm}
The number of distinct words that can be represented using a word embedding is called its word representation capacity. Size of the BoW embeddings grows linearly with the size of the vocabulary. However, the size of RI space, i.e. the number of possible distinct context vectors, is given by
\begin{equation}
N(n,r,K) = \    ^{n}C_r\times K^r 
\label{wordsize}
\end{equation}where, $n$ is the size of the context vector, $r$ is the number of non-zero entries and K is the possible choices for a non-zero entry. When  $r = 1$ the typical choice of non-zero value is $1$; when $r = 2$ the values are from the set $\{+1, -1\}$. Thus the word representation capacity of RI-based representations increases exponentially with $n$ keeping $K$ and $r$ fixed as shown in Fig \ref{fig1a},\ref{fig1b}. This allows one to accommodate a much larger vocabulary with the same vector size.  For illustration, an one-hot embedding of vector size $1000$ can accommodate a vocabulary of maximum $1000$ word.  But even with  $K = 2$  and $r = 1$  its capacity is  $999 \times 500  = 4,99,500$.   If  $r  = 2$, the capacity increases to $19,98,000$.
\vspace{-0.3cm}
\subsection{Probability of orthogonality}
\label{subsec :2.3}
The probability of orthogonality is the probability that two randomly selected context vectors in the $n$ dimensional space will have the inner product as zero. As shown in \cite{Chatterjee2013EffectON}, the deviation from the classical one-hot representation implies compromise on the probability of orthogonality. One may observe that accuracy of the representation in conveying meaningful semantic information decreases with the decrease in the probability of orthogonality (explained in Section \ref{results}). Hence the aim is to maximize the probability of orthogonality of two randomly selected context vectors from the RI space while decreasing the dimension of the context vectors to represent semantic information, and also to preserve word relationship information accurately. When the non-zero entries can only take the value $+1$, the probability of orthogonality for $n,r$ is given by
\begin{equation}
 P_{ortho}= \frac{^{n-r}C_r}{N(n,r,1)}
\end{equation}
Here the numerator denotes the number of possible choices of orthogonal vector for a given vector $v$ where $r$ non-zero entry indices are common with $v$. The denominator denotes the total number of choices given by $N(n,r,1)$. Note when $r=1$, it boils down to one-hot encoding case and the vectors are always orthogonal to each other. However, in the RI case with $r=2$ i.e. choice of non-zero entry being $+1,-1$, the probability can be calculated by enumeration of vectors only for small values of r. However, these are readily obtained by Eqn. (\ref{portho}) \footnote{A more generic and rigorous treatment is provided in Section 4.1.}. These are special cases of Eqn. (\ref{portho}) where $r_{1}=r_2=r$ and $P_{r}(r)=1$. Here N(n,r,K) is from Eqn. (\ref{wordsize})
 \begin{equation}\nonumber
    r=2, K=2 : P_{ortho}= \frac{^{n-2}C_2\times2^2 +2}{N(n,2,2)}
\end{equation}
\begin{equation}\nonumber
    r=3, K=2 : P_{ortho}= \frac{^{n-3}C_3\times2^3 +^3C_2\times^{n-3}C_1\times2\times2}{N(n,3,2)}
\end{equation}
\begin{equation}\nonumber
    r=4, K=2 : P_{ortho}= \frac{^{n-4}C_4\times2^4 +^4C_2\times^{n-4}C_2\times2^2\times2 + 6}{N(n,4,2)}
\end{equation}
\begin{equation}\nonumber
     r=5, K=2 : P_{ortho}= \frac{^{n-5}C_5\times2^5 +^5C_2^{n-5}C_3\times2^3\times2 + ^5C_4^{n-5}C_1\times2\times6}{N(n,5,2)}
\end{equation}
\begin{equation}\nonumber
     r=6, K=2 : P_{ortho}= \frac{^{n-6}C_6\times2^6 +^6C_2^{n-6}C_4\times2^4\times2 + ^6C_4^{n-6}C_2\times2^2\times6+20}{N(n,6,2)}
\end{equation}
The values are plotted in  Fig.\ref{fig2a}, \ref{fig2b}. It can be observed that for a given n the probability of orthogonality decreases as $r$ increases. Also, these figures show that the probability saturates at a faster rate ($\geq 0.9$) for small values of $r$, and after that the increase in probability is marginal. The cut-off values of $n$ are obtained against probability of orthogonality for 
\begin{itemize}
    \item Standard RI case with entries randomised with only $+1$
    \item Standard RI case with non-zero entries taking values in $\{+1, -1\}$.
\end{itemize} It was observed while for a given probability of orthogonality, $n$ is marginally lower for RI case but with a much larger word capacity which motivates randomization of entries between $\{+1, -1\}$.

\begin{subfigures}
\begin{figure}[t]
   \begin{minipage}{0.48\textwidth}
     \centering
     \includegraphics[width=\linewidth]{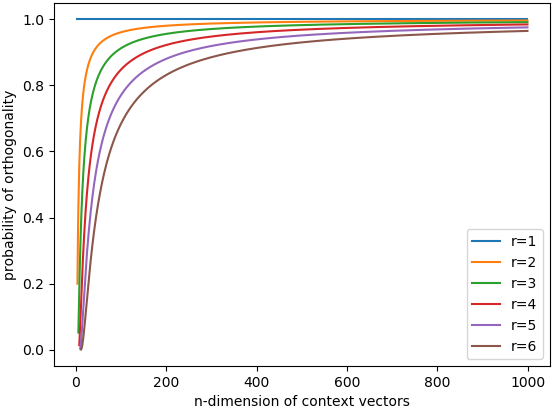}
      \caption{Probability of orthogonality of 2 vectors randomly selected in the RI space where each non zero entry is 1}
    \label{fig2a}
   \end{minipage}\hfill
   \begin{minipage}{0.48\textwidth}
     \centering
     \includegraphics[width=\linewidth]{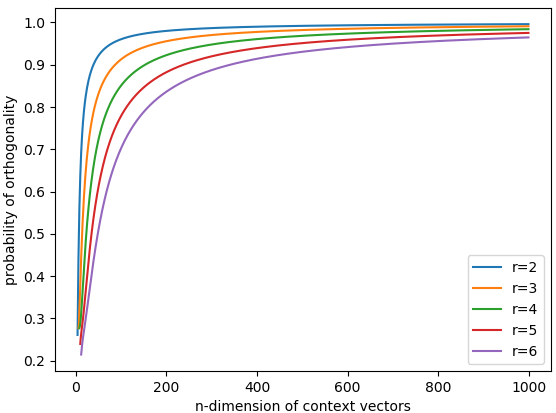}
     \caption{Probability of orthogonality of 2 vectors randomly selected in RI space where non zero entry can be 1 or -1}
    \label{fig2b}
   \end{minipage}
\end{figure}
\end{subfigures}
\vspace{-0.5cm}
\subsection{Probability of orthogonality of a randomly chosen subset}
In actual practice, to represent a vocabulary that contains W distinct words, a subset of the set of all the vectors formed from Random Indexing of $n$ dimensional context vectors with $r$ non-zero entries is chosen. This is done to ensure a rather high probability of orthogonality and to account for any new context vectors that may be added to the vocabulary. To motivate the choice of n and r for practical purposes the graphs in Fig. 3a, 3b depict the probability of orthogonality in a randomly chosen subset vs. the size of the representation capacity for fixed values of $n, r$. It can be observed that initially the probability fluctuates but it becomes stable after sample size reaches 200. The steady state value increases with $n$ for a given value of $r$. The above figures show that that by decreasing the size of the RI-space below a threshold does not increase the probability of orthogonality for a randomly chosen subset of the context vectors. So decreasing the size of RI-space keeping the dimension of vectors high, is shown to be non efficient. Hence the values of n and r are to be chosen in a way that maximises the probability of orthogonality while minimising the dimension of context vectors.

\begin{subfigures}
\begin{figure}[t]
   \begin{minipage}{0.48\textwidth}
     \centering
     \includegraphics[width=\linewidth]{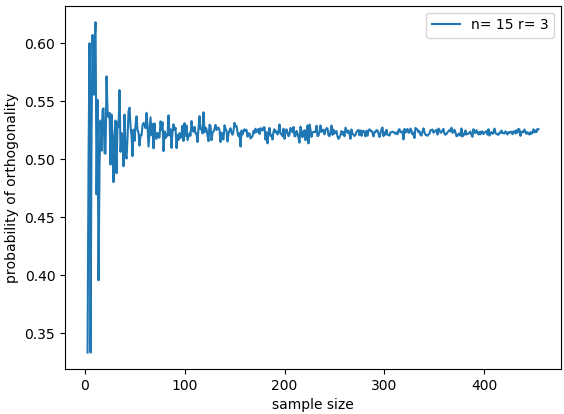}
      \caption{Probability of orthogonality of a subset for n=15, r=3}
    \label{fig3a}
   \end{minipage}\hfill
   \begin{minipage}{0.48\textwidth}
     \centering
     \includegraphics[width=\linewidth]{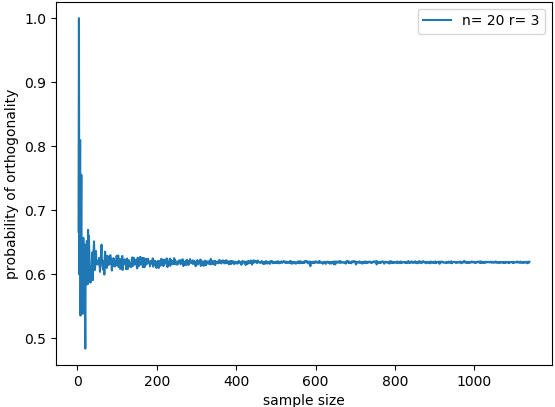}
     \caption{Probability of orthogonality of a subset for n=20, r=3}
    \label{fig3b}
   \end{minipage}
\end{figure}
\end{subfigures}

\vspace{-0.3cm}
\section{Introducing a Probability Distribution on r}
This section provides a broader framework for Random Indexing based approaches. Here we examine the effect of introducing a probability distribution of the number of non-zero entries. It can be seen from Fig. 2a, 2b that for a given value of $n$ as the value of $r$ is increased, the probability of orthogonality decreases, but the representation capacity increases. This motivates the idea of a non-constant non-zero number of entries in the representation. Therefore instead of choosing a fixed value of $r$, we propose a distribution on one the values of $r$. We thus try to randomize the number of non-zero entries by defining a probability distribution which takes into account the following:
\begin{itemize}
    \item For a given value of $n$ lesser the value of $r$ higher is the probability of orthogonality
    \item Higher value of $r$ has more word representation capacity.
\end{itemize}   Thus, in a dynamic case where the word space is continuously evolving, introducing a distribution on r takes care of the probability of orthogonality and total number of words that can be accommodated.
\subsection{Probability of orthogonality of two randomly indexed vectors of length n with $r_1$ and $r_2$ non-zero entries }
Let $v_1$  and $v_2$ be the vectors of length $n$ with number of non-zero entries $r_1$  and $r_2$, respectively. Without loss of generality let $r_2 \geq r_1$ and $n\geq r_2 + r_1$. Then the number of ways in which a vector can be chosen with $r_1$ non-zero entries to make dot product zero be denoted by $\eta(r_1 , r_2)$.
This can be broken down in the following two cases: \\
\textbf{Case 1 :} after choosing $r_2$ indices in $v_2$, the $r_1$ indices in $v_1$ are chosen from the remaining $n-r_2$ places after is \begin{equation}\nonumber
    ^{n-r_2}C_{r_1}\times2^{r_1}
\end{equation}
\textbf{Case 2 :} $2k$ (where $1\leq k\leq r_1/2$ if $r_1$ is even , else $k \leq r_1-1/2$) common places are chosen out of $r_2$ and the remaining $r_1-2k$ out of $n-r_2$  places, then the number of such combinations are given by 
\begin{equation}\nonumber
    \phi(2k, r_2)\times2^{r_1-2k}
\end{equation}
Where $\phi(2k, r_2)$ is the number of combinations where $2k$ entries are common with $r_2$ entries and the dot product is zero. Since the dot product is zero, it must be an even sum of +1 and -1 (as these are the only two possibilities of the product of non-zero entries). Since there are $2k$ such products exactly k must be +1 and -1. So the problem boils down to choosing k entries out of $r_2$ and making the product of the corresponding terms +1 and the other k as -1. For a given vector $v_2$ making +1 and -1 for each entry has only 1 possibility. This gives the following two results :
\begin{equation}
    \phi(2k, r_2) =  ^{r_2}C_{2k}\times^{2k}C_k
\end{equation}
\begin{equation}
    \eta(r_1 , r_2) =   \sum_{k}\phi(2k, r_2)\times2^{r_1-2k}\times ^{n-r_2}C_{r_1-2k}  \ \ \   \ \ 0\leq k \leq \lfloor r_1/2\rfloor
\end{equation}
Also the total number of ways of choosing the vector with $r_1$ non-zero entries is $^{n}C_{r_1}\times2^{r_1}$ which is taken care of by allowing k to be 0. Hence, the probability of orthogonality is given by : 
\begin{equation}
P_{ortho}(r_1, r_2 \vert n) = \frac{\eta(r_1 , r_2)}{N(n,r_1,2)}
\label{orthoprob}
\end{equation}

This gives a class of RI representations based on the underlying probability distribution imposed on $r$. For varying discrete distributions imposed on r taking values from the set S, a number of Random Indexing based models can be obtained. 
\par Furthermore, each of these discrete distributions can be implemented by making discretized bins in the uniform distribution over an interval and mapping them to a discrete value of $r$ in the set S. Let us denote the probability of $r$ taking the value $\alpha$ to be $P_r(\alpha)$. Then denoting $P(v_1, v_2)$ as the probability of orthogonality of two vectors $v_1$, $v_2$ with $r$ taking values from the set $S =$ \{$r_1$, $r_2$, ..., $r_n$\}  where $r_i \in \mathbb{N} \ \ \forall i \geq 1 $, the expected value of the probability of orthogonality is given by  
 \begin{equation}
     \mathop{\mathbb{E}}(P(v_1, v_2)) = \sum_{r_i\in S, r_j\in S} P_r(r_i)\times P_r(r_j)\times P_{ortho}(r_i, r_j \vert n) 
     \label{portho}
 \end{equation}
where $P_{ortho}(r_i, r_j)$ is the probability of orthogonality given two context vectors of length n having $r_i$ and $r_j$ non-zero entries, respectively. $P_r(r_i)$ denoting the probability of $r$ taking value $r_i$ in set $S$. The results for deterministic case presented in Section 2.3 are degenerate cases of Eq.(\ref{portho}) where $r_1=r_2=\alpha$ and $P_r(\alpha)=1$. Hence $\mathop{\mathbb{E}}(P(v_1, v_2))= P_{ortho}(\alpha,\alpha \vert n).$

\subsection{Representation Capacity and Probability of orthogonality}
Here we illustrate the probability of orthogonality and representation capacity of the Probabilistic RI case when restricting the distribution of $r$ to a uniform discrete distribution over the set $S=\{2,3,4,5,6\}$. We plot the representation capacity and the expected probability of orthogonality $\mathop{\mathbb{E}}(P(v_1, v_2))$ given by Eqn. (\ref{portho}). These plots are shown in Fig. \ref{fig4a} and \ref{fig4b} respectively. It is observed that the representation capacity curve for Probabilistic RI case runs closest to the largest representation capacity, i.e. the representation capacity for $r=6$. This further motivates the idea of distributing r probabilistically. 
\par Another insightful observation is that Probabilistic RI gains significantly on the representation capacity while not compromising on the probability of orthogonality by uniformly distributing r in the set $\{2,3,4,5,6\}$ i.e. each context vector can have $2$ or $3$ or $4$ or $5$ or $6$ non-zero entries with equal probability. An n-dimensional context vector with $r$ many non-zero entries will have zeros at $n-r$ indices and non-zero entries (+1 or -1) at the remaining $r$ places as given in Eqn. (\ref{context}). Hence the proposed representations further alleviate the problem of large corpus sizes that change dynamically as the Probabilistic RI space can accommodate large number of context vectors without compromising on the probability of orthogonality. 

\begin{subfigures}
\begin{figure}[b]
   \begin{minipage}{0.48\textwidth}
     \centering
     \includegraphics[width=\linewidth, height=3.5cm]{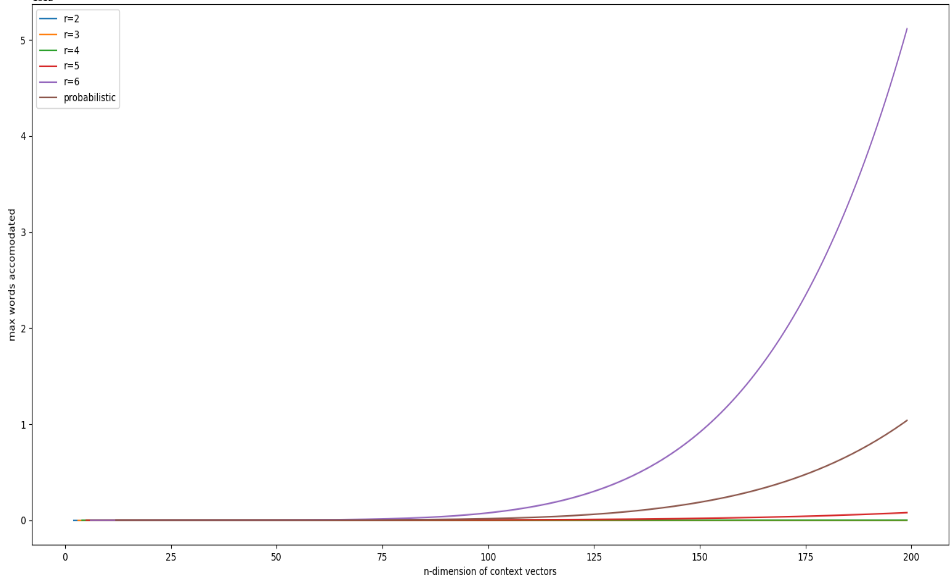}
     \caption{Expectation of word capacity}
    \label{fig4a}
   \end{minipage}\hfill
   \begin{minipage}{0.48\textwidth}
     \centering
     \includegraphics[width=\linewidth, height=3.5cm]{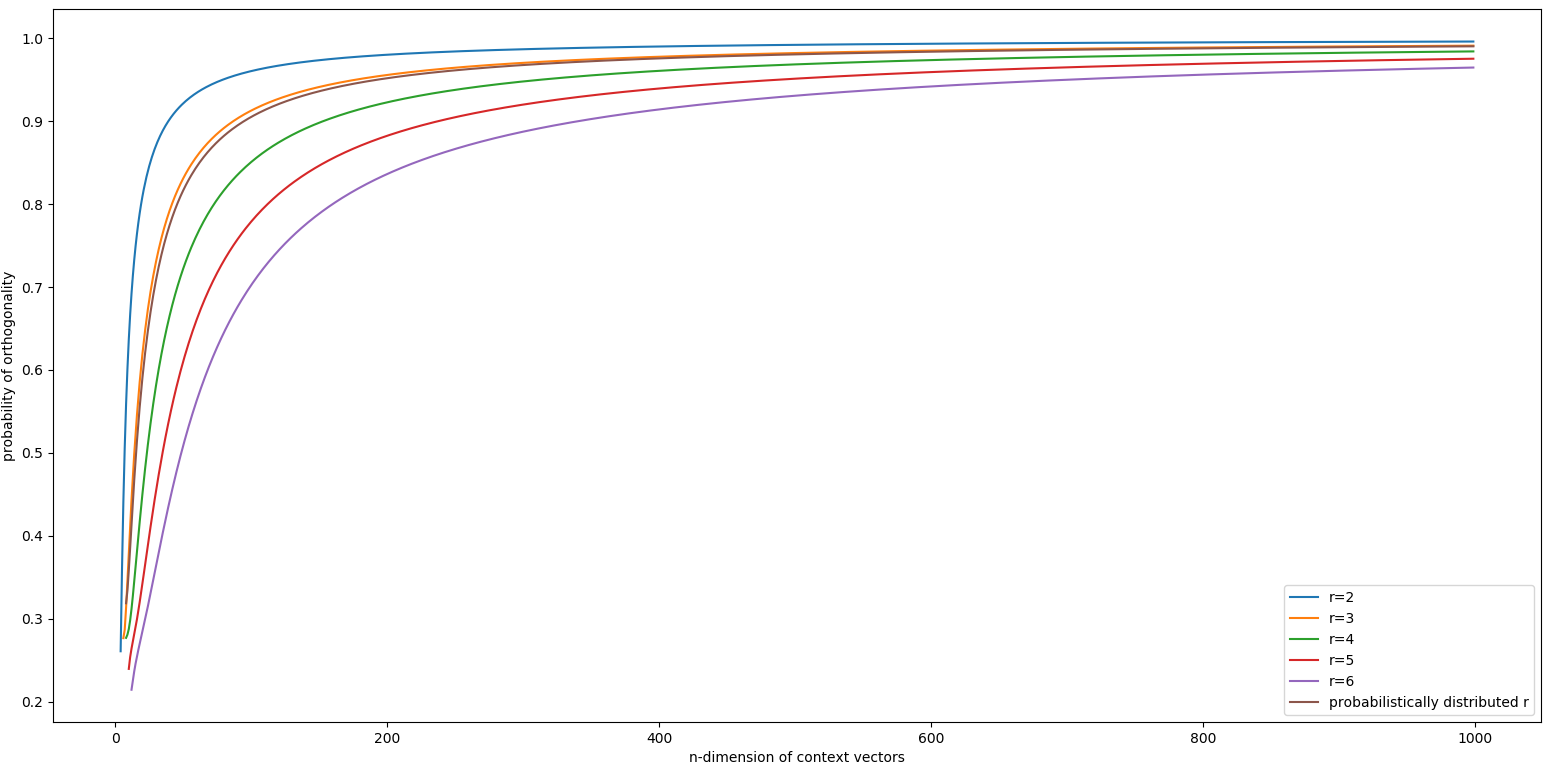}
     \caption{Expectation of prob. of orthogonality}
    \label{fig4b}
   \end{minipage}
\end{figure}
\end{subfigures}

\section{Event Detection}
Event detection through sources like tweets is useful for two reasons. Firstly, since these sources are not regulated, a variety of events can be detected that may otherwise not be shown by the mainstream media due to biases. Secondly, these sources reflect the event occurrence much faster than the newspapers or media houses, who normally take time to curate their content. Hence in this work we use twitter data. i.e. tweets, for three different event detection scenarios as given in Sec. 6. 
\\ Since in language data the closely associated words usually occur in the same context, i.e. close to each other, their semantic vectors are closely embedded in the n-dimensional space. For this work a context window of five words on either side of the key word is considered to be the context of a word as given in Eqn. (3). For the boundary words, the context is taken only in one direction. 
\par In a dynamic scenario with the change in the language data with time, the semantic vectors associated with the words change displaying a semantic shift that can be used to determine the changing association between a query word with respect to the other words of the corpus. Since the dimension of the context vectors and the semantic vectors are not altered in the RI-based representations, a subset of semantic vectors can be obtained spanning the time period of interest and are subjected to further analysis to detect the occurrence of an event.

\subsection{Tracking Semantic Shift to Detect Event}
Given a query word $w_q$ the semantic shift of the vectors associated with other words with respect to $w_q$ is tracked. This is possible as the length of the semantic vector remains constant in time. This allows one to circumvent the problem of new words being introduced and words being deleted in the vocabulary at different instants of time. The values of n and r are so chosen that the word representation capacity is greater than N i.e. the size of vocabulary W to ensure a high probability of orthogonality for representation of data set. As the size of the RI space grows exponentially with n (dimension of context vectors), a random subset S of the total RI space is chosen such that $ \vert  S \vert  =   W \vert $. Context vectors are then assigned to each distinct word in the corpora. A bijective map is created from the vocabulary to random subset chosen $M : W\rightarrow S$. The context vectors are stored in a dictionary according to this map. This allows one to add context vectors to calculate semantic vectors without having to traverse them explicitly. Thus a context vector C\textsubscript{i\  }for word $w_i$ in word space is a mapping \textbf{$C_i: N^r\rightarrow \{-1,1\}^r$}. To calculate semantic vectors, the word set is traversed once. The angle between semantic vectors of word $w_i$ and $w_j$ is calculated as :
\begin{equation}
\label{theta}
\theta_{i,j} = \cos^{-1}(\frac{\mathbf{v_i}\cdot\mathbf{v_j}}{\vert\vert\mathbf{v_i}\vert\vert\cdot\vert\vert \mathbf{v_j}\vert\vert})    
\end{equation}
where $\mathbf{v_i}, \mathbf{v_j}$ are the semantic vectors corresponding to the words $w_i, w_j$, respectively.
\subsection{Algorithm for suggesting words related to event}
This section provides the proposed algorithm to detect an event relating to a particular query word $w_q$, by tracking the semantic shift of the words with respect to $w_q$ corresponding to time instances $t$ and $s$, where $t < s$ . This is done by ranking the words according to a coefficient  and compiling a list of words that are most relevant to the event with respect to the query word $w_q$. Let $W_t$ and $W_s$ be the vocabulary at the time instance $t$ and $s$, respectively. The following steps explain the algorithm. 

\begin{enumerate}
\item Sort the list of words in $W_t$ in ascending order of the absolute value of the angular difference measured using cosine similarity from the query word $w_q$ to get $V_t$. Sort $W_s$ in a similar way to get $V_s$.
\item Calculate normalized frequencies $f_t(w)$, $f_s(w)$ of the words in both the data sets $W_t$ and $W_s$ independently as follows, where $\mathbbm{1}$ is the indicator function : 
\begin{equation}
    f_{\tau}(w) = \frac{\sum_{v \in W_t} \mathbbm{1}_{v = w}}{\vert W_t \vert}   
\end{equation}
where $\tau=\{t,s\}$. This is done to ensure words occurring more frequently do not dominate the coefficient.
\item Select a value for  a hypermeter $\kappa$, where $\kappa$ denotes the maximum number of  words to be considered for analysis from the  post data set $W_s$. Let $\mathcal{W}$ be the set comprising the top $\kappa$ words  $w_1, \dots, w_{\kappa}$. 
\item Calculate suggestion coeff $S(w)$for each word $w \in \mathcal{W}$ as 
\begin{equation}
 S(w) = \Delta \times  \frac{rank_{V_t}(w)}{rank_{V_s}(w)} \times \frac{f_t(w)}{f_s(w)}  
 \label{rank}
\end{equation}
$\Delta = \vert \theta_{post} - \theta_{pre}\vert$ where $\theta$ is computed using Eqn. (9) for word $w, w_q$ in both $W_s, W_t$. \\
\item Sort $\mathcal{W}$ in decreasing order of $S(w)$.
\item Filter out words $w$ such that $rank_{V_t(w)} \leq \kappa$ and report top-$P$ words.
\end{enumerate}

\hspace{-0.5cm}The suggestion coefficient takes into account the following :
\begin{itemize}
    \item The words with larger semantic shifts are given more weight.
    \item The words which became closer to the query word at the later time $s$ (i.e. they were not closer at earlier time $t$)  are given more weight. This is done so that the words that continue to be closer are filtered out as their ranks will not change much in the sorted list.
    \item The words which were infrequent at time $t$, but frequent at time $s$ are given more weights. This is because those which started appearing frequently later are more likely to be related to some event.
\end{itemize}

\begin{theorem}
The proposed event detection algorithm is log linear in  the size of vocabulary.
\end{theorem}
\begin{proof}
Let $n$ be the dimension of context vectors, $r$ be the number of non-zero entries, $m$ be the context range and \textbf{$ \vert W\vert $} be the size of vocabulary. Firstly generating $\vert W \vert$  context vectors which is \textbf{$O(r\vert W\vert )$}. Then $ \vert W\vert $ semantic vectors are generated which involves adding at most $2m$ context vectors for each word, hence is \textbf{$O(2 m\vert W \vert )$}. The calculation of angle $\theta$ for each word $w \in W$ is again $O(n)$ as max number of non-zero entries can be $n$, and one semantic vector is traversed once for a dot product. The sorting of all the semantic vectors is O($\vert W\vert\log\vert W\vert$). The sorting of the final list is O($\kappa\log \kappa$). So overall time complexity for a given $m$ is $O(r\vert W \vert+2 m  \vert  W \vert + \vert W\vert\log\vert W\vert +\kappa\log \kappa +n) = O( \vert W\vert\log\vert W\vert )$ as the constants $n,r,m, k$ are small as compared to $ \vert W \vert $ i.e. $N$.
\end{proof}
 Hence this runs in log linear time of vocabulary. Note that when the code is run different times, the angles for probabilistic RI representations may be different, this is because a random subset of RI space is chosen each time as context vectors and this may differ. But the general trend in change of semantic vectors is same hence can be used for event detection.

\section{Experimental Setup}
For experimentation the following events were chosen : 
\begin{enumerate}
    \item `launch of iPhone X'  with $t = Oct \ 2017, \ s = Nov \ 2017$
    \item `demonetization in India'  with $t = Oct \ 2016,\  s = Nov \  2016$
    \item `corona virus outbreak'  with $t = Oct \ 2019,  \ s = May  \ 2020$
\end{enumerate} Each dataset comprises of tweets before the event and after the event with time slice corresponding to one month. These tweets were imported using the Python library \texttt{GetOldTweets3} \footnote{https://pypi.org/project/GetOldTweets3/} and setting the area filter to New York and New Delhi in order to avoid tweets of non-English languages. Top tweets for the month are imported for each time instant $t,s$ respectively. The proposed algorithm is run on the pair of time instants $(t, s)$ to report suggested words post-event. Since there is no beginning time instant for the pre-event dataset, words are suggested based on the normalized frequency and angular difference of semantic vectors measured from the query word in the data set at time instant $t$. We report top-5 ranked suggested words for the query word `India', `virus' and `iPhone'. 
\par Additionally, for testing the robustness of the method with respect to the scale of the dataset, we import the following two types of datasets for analysis of `launch of iPhone X' event: 
\begin{itemize}
    \item Data Set(small/clean) :  number of distinct words $ \sim $ 600, imported by taking top tweets, cleaner data set.
    \item Data Set(large/noisy) : number of distinct words $ \sim $ 1700, imported by taking non-top tweets, random tweets also included which add noise.
\end{itemize}
We present results of suggested words for both the datasets to gauge resistance to noise and scalability. Further, we also compare the method with the Random Indexing method proposed in \cite{Chatterjee2013EffectON}.
\subsection{Data Preprocessing}

The tweets retrieved are first converted to string for further processing by removing the following.

\begin{enumerate}
\item \textit{URLs and mentions: }URLs contain the key character sequences $``$http$"$  and $``$.com$"$ , mentions contain ``@'' these are used to identify links, usernames and remove them as they do not provide any meaningful information in our case.


\item \textit{$\#$ tags and other special characters: }$\#$ tags and other characters such as ! . , $``$ $\&$  ] [ etc. are commonly used which mask the true meaningful words, these are replaced with a blank space wherever found in the string to converge to the core word.

\item \textit{Stop Words:} These are a part of English sentences but do not provide any meaningful observation. They are removed from the data set using \textit{nltk }library of python and importing English stop words, so as not to add these words to our final data set if they belong to the set of stop words.
\end{enumerate}

\begin{table*}[h]

 \label{table3}\centering\caption{Top-5 suggested words related to the event concerning the query word ``India''}
 \begin{adjustbox}{max width=\textwidth}
\begin{tabular}{p{1.79in}p{1.79in}p{1.79in}}
\hline
\multicolumn{1}{|p{1.79in}}{} & 
\multicolumn{1}{|p{1.79in}}{\textbf{Pre-event}  } & 
\multicolumn{1}{|p{1.79in}|}{\textbf{Post-event}} \\
\hline
\multicolumn{1}{|p{1.79in}}{\textbf{Classical/BoW \newline $n= \sim  600, \sim  1700$}} & 
\multicolumn{1}{|p{1.79in}}{\textit{`sardarpatel', `delhi', `NDTV', \newline	`diwali', 	`pakistan'}} & 
\multicolumn{1}{|p{1.79in}|}{\textit{`demonetization', 	`modi', `currency', \newline	`blackmoney',	`rupee'}} \\
\hline
\multicolumn{1}{|p{1.79in}}{\textbf{Deterministic RI  \newline $n=12, r=6$}} & 
\multicolumn{1}{|p{1.79in}}{\textit{`delhi', `diwali',	`china', \newline	`travel',	`NDTV'}} & 
\multicolumn{1}{|p{1.79in}|}{\textit{`demonetization', `currency',\newline `powermove',  `ban', `modi'}} \\
\hline
\multicolumn{1}{|p{1.79in}}{\textbf{Probabilistic RI } \newline r distributed uniformly in \newline $ \{ 2,3,4,5,6\}  , n=12$} & 
\multicolumn{1}{|p{1.79in}}{\textit{`diwali', `sardar patel', `NDTV', \newline	`hockey', `delhi'}} & 
\multicolumn{1}{|p{1.79in}|}{\textit{`demonetization', `modi',	`rupee' \newline	`jaihind',	`ban'}} \\
\hline

\end{tabular}
\end{adjustbox}

 \label{table4}\centering\caption{Top-5 suggested words related to the event concerning the query word ``virus''}
 \begin{adjustbox}{max width=\textwidth}
\begin{tabular}{p{1.79in}p{1.79in}p{1.79in}}
\hline
\multicolumn{1}{|p{1.79in}}{} & 
\multicolumn{1}{|p{1.79in}}{\textbf{Pre-event}  } & 
\multicolumn{1}{|p{1.79in}|}{\textbf{Post-event}} \\
\hline
\multicolumn{1}{|p{1.79in}}{\textbf{Classical/BoW \newline $n= \sim  600, \sim  1700$}} & 
\multicolumn{1}{|p{1.79in}}{\textit{`influenza', `HIV', `computer', \newline	`slenderman', 	`hepatitis'}} & 
\multicolumn{1}{|p{1.79in}|}{\textit{`corona',	`pandemic',	`vaccine', \newline	`covid',	`china'}} \\
\hline
\multicolumn{1}{|p{1.79in}}{\textbf{Deterministic RI  \newline $n=12, r=6$}} & 
\multicolumn{1}{|p{1.79in}}{\textit{`HIV', `influenza',	`trump', \newline	`ebola',	`computer'}} & 
\multicolumn{1}{|p{1.79in}|}{\textit{`corona',	`covid', `quarantine', \newline	`pandemic',	`ban'}} \\
\hline
\multicolumn{1}{|p{1.79in}}{\textbf{Probabilistic RI } \newline r distributed uniformly in \newline $ \{ 2,3,4,5,6\}  , n=12$} & 
\multicolumn{1}{|p{1.79in}}{\textit{`HIV', `trump', `computer', \newline	`ebola', `influenza'}} & 
\multicolumn{1}{|p{1.79in}|}{\textit{`corona',	`vaccine',	`pandemic',	\newline `covid', 	`quarantine'}} \\
\hline

\end{tabular}
\end{adjustbox}
\end{table*}

\begin{table*}
 			\label{table1}\centering\caption{Top-5 suggested words post event for query word ``iPhone'' on clean and noisy dataset}
\begin{adjustbox}{max width=\textwidth}
\begin{tabular}{p{1.79in}p{1.79in}p{1.79in}}
\hline
\multicolumn{1}{|p{1.79in}}{} & 
\multicolumn{1}{|p{1.79in}}{\textbf{Data Set Small/Clean}  } & 
\multicolumn{1}{|p{1.79in}|}{\textbf{Data Set Large/Noisy}} \\
\hline
\multicolumn{1}{|p{1.79in}}{\textbf{Classical/BoW \newline $n= \sim  600, \sim  1700$}} & 
\multicolumn{1}{|p{1.79in}}{\textit{`x', `shoot', `mode',\newline `shotoniphonex', `iphone$ \ldots $ '}} & 
\multicolumn{1}{|p{1.79in}|}{\textit{`x', `photo', `portrait',\newline `7', `que'}} \\
\hline
\multicolumn{1}{|p{1.79in}}{\textbf{Deterministic RI  \newline $n=12, r=6$}} & 
\multicolumn{1}{|p{1.79in}}{\textit{`x', `new', `iphone$ \ldots $ ',\newline `shoot', `potrait'}} & 
\multicolumn{1}{|p{1.79in}|}{\textit{`x', `shot', `puts',\newline `photography', `plus'}} \\
\hline
\multicolumn{1}{|p{1.79in}}{\textbf{Probabilistic RI } \newline r distributed uniformly in \newline $ \{ 2,3,4,5,6\}  , n=12$} & 
\multicolumn{1}{|p{1.79in}}{\textit{`x', `plus', `stay', \newline`love', `iphonex'}} & 
\multicolumn{1}{|p{1.79in}|}{\textit{`x', `shoot', `mode',\newline `shotoniphonex', `new'}} \\
\hline

\end{tabular}
\end{adjustbox}

\centering\caption{Angle between semantic vectors of top suggested words and ``iPhone'' pre and post event}\label{table3}
\begin{adjustbox}{max width=\textwidth}
\begin{tabular}{|l|l|l|l|l|}
\hline
words &
  \begin{tabular}[c]{@{}l@{}}Data Set Small\\ Pre-event\end{tabular} &
  \begin{tabular}[c]{@{}l@{}}Data Set Small\\ Post-event\end{tabular} &
  \begin{tabular}[c]{@{}l@{}}Data Set Large\\ Pre-event\end{tabular} &
  \begin{tabular}[c]{@{}l@{}}Data Set Large\\ Post-event\end{tabular} \\ \hline
\multicolumn{5}{|l|}{\textbf{Baseline}, $n=600 $ for Data set small, $n=1700$ for Data set large}                                                            \\ \hline
8               & 0.83             & 1.06             & 0.75            & 0.79            \\ \hline
x               & 1.28             & 0.83             & 1.28            & 0.72            \\ \hline
\multicolumn{5}{|l|}{\textbf{Deterministic RI} $r=6, n=12$}                                \\ \hline
8               & 0.53             & 0.61             & 0.66            & 0.51            \\ \hline
x               & 0.81             & 0.43             & 1.06            & 0.36            \\ \hline
\multicolumn{5}{|l|}{\textbf{Probabilistic RI} $r$ distributed uniformly in $\{2,3,4,5,6\}, n=12$} \\ \hline
8               & 0.66             & 0.57             & 0.50            & 0.42            \\ \hline
x               & 0.94             & 0.47             & 0.77            & 0.33            \\ \hline
\end{tabular}
\end{adjustbox}

 			\label{table2}\centering\caption{Time taken in milli Seconds for 10 iterations of algorithm for each representation}
\begin{adjustbox}{max width=\textwidth}
\begin{tabular}{|l|l|l|l|l|l|l|}
\hline
                                                                     & \multicolumn{3}{l|}{Data Set Small} & \multicolumn{3}{l|}{Data Set Large} \\ \hline
context range m                                                      & 2          & 5         & 8          & 2          & 5          & 8         \\ \hline
\begin{tabular}[c]{@{}l@{}}\textbf{Classical/BoW}\\ $n=\sim 600, \sim 1700$\end{tabular}  & 50.36      & 73.10     & 127.61     & 441.27     & 805.32     & 993.08    \\ \hline
\begin{tabular}[c]{@{}l@{}}\textbf{Deterministic RI}\\ $n=12, r=6$\end{tabular} & \textbf{35.62}      & 68.57     & 57.52      & 123.50     & \textbf{191.51}     & 262.16    \\ \hline
\begin{tabular}[c]{@{}l@{}}\textbf{Probabilistic RI} \\ r distributed uniformly in \\$\{2,3,4,5,6\}, n=12$ \end{tabular} & 45.87 & \textbf{44.08} & \textbf{47.17} & \textbf{110.31} & 194.97 & \textbf{222.04} \\ \hline
\end{tabular}
\end{adjustbox}

 \end{table*}

\vspace{-0.6cm}
\subsection{Results}
\label{results}
The following categories of experiments were performed for the proposed method:
\begin{enumerate}
    \item Experiments to validate the quality of words suggested by the proposed method and measure the robustness with respect to noise in the dataset
    \item Experiments demonstrating the efficacy of lesser dimensional Probabilstic RI embeddings with respect to computation time while being able to encode information efficiently
    \item Ablations to demonstrate superiority of the Probabilistic RI representations in depicting semantic shift as compared to other RI based representation 
\end{enumerate}To set a baseline for tracking the shift in semantic vectors of different words, context vectors are assigned as given by  Eqn. (\ref{context}) and semantic vectors are calculated using Eqn. (\ref{semantic})  for context(m) ranging from 1 to 10.  The angle between semantic vectors of different words and query word are calculated in pre and post data sets. Here we report the top 5 words ranked according to Eqn. (\ref{rank}) i.e. most relevant to the event happening with a context range of 5 and time taken in milliseconds taken by each representation method averaged over 10 iterations. 

\par Table 1 reports the results for the event `demonetization in India', with the query word as \textit{India}. Here one can observe the shift in the semantic vectors with respect to the query word. The top suggested word being \textit{demonitization}, the rest of the suggested words also reflect the other relevant aspects of the event, varying with each representation method. Table 2 presents the results on the event `corona virus outbreak' with the query word as \textit{virus}. It can be readily seen that before the pandemic broke out the suggested words such as, \textit{HIV, influenza}, show the common viruses prevalent prevalent before the corona broke out. However after the event the suggested words change drastically to \textit{corona, pandemic, vaccine} with almost none of them being actual viruses. This shift in suggested words accurately depicts the pandemic. 

\par In another set of experiments the effect of cleaned data and noisy data on the shift is studied. This is shown in Table 3 with the query word being \textit{iPhone}. It is observed that before the launch query word is closer to the word \textit{8} as people tweet about the iPhone 8 that is the prevailing iPhone version and post launch becomes closer to \textit{x} owing to the launch of the new version. We find that the algorithm correctly predicts words such as \textit{x}, \textit{new, portrait}. From Table 4, it can be seen for the query word \textit{iPhone} in both the clean and noisy datasets that the proposed method is able to track meaningful semantic shifts over time as the vocabulary evolves and is robust to noise.

\par It is noteworthy that while BoW representation method requires $600, 1700$ dimensional embeddings (as in Table 4) for the small and large data set respectively, the Probabilistic RI representations with just 12 dimensions are able to compete owing to the appropriately chosen value of r. The running times reported in Table 5 also show that the proposed probabilstic embeddings run up to  four times faster than BoW as the size of corpus increases by three times while maintaining semantic accuracy. We hypothesise this gain to be much larger in real time datasets where representations for a large corpus of words are to be maintained. For instance if $250000$ unique words are to be encoded, the one-hot embeddings would be of $250,000$ dimensions, however for a probability of orthogonality of $0.97$ with $r=3$, only $350$ dimensional embeddings will suffice (See Table 6). This makes the proposed representations highly scalable. 

\par We also plot the effect of probability of orthogonality on the accuracy of the representations to convey meaningful semantic information about the relationship between different words in a given dataset. To examine the accuracy with respect to the base case (BoW representations), probability of orthogonality is varied by varying ``n'' for a given``r''. The results are averaged over 20 iterations for all cases.
The metric used is :  $1-\vert \theta_i -  \theta_0\vert$, where
$\theta_i$ is the angle between initial and final semantic vectors of \textit{iPhone}  in RI case, 
$\theta_0$ is the angle between initial and final semantic vectors of \textit{iPhone} in base case. This is shown in Fig. \ref{figy1}, \ref{figy2}, where a cubic trendline has also been plotted. It can be noticed that with increasing the probability of orthogonality the accuracy of the representations increases.

\par An ablation experiment is done to compare the Probabilstic RI method to that proposed in \cite{Chatterjee2013EffectON} for the event `launch of iPhone X'. This is done to study the effect of cancellation of the non-zero entries while summing up the context vectors to get semantic vectors. While the representations in \cite{Chatterjee2013EffectON} forego this cancellation, the proposed representations do not. For this purpose, we plot the change in the angle between words \textit{X} and \textit{iPhone} pre and post event to observe semantic shift. These are shown in Fig. \ref{figz1}, \ref{figz2}. It can be observed that the Probabilistic RI representation performs better in depicting the semantic shift as the amount of deviation is larger in this case although somewhat noisy. This is because of the fact the +1's and -1's are distributed randomly instead of upper and lower halves respectively as in \cite{Chatterjee2013EffectON}. Hence, the proposed representations are found to be better suited for the algorithm to detect a significant semantic shift to suggest words.

\begin{subfigures}
\begin{figure}[h]
   \begin{minipage}{0.45\textwidth}
     \centering
     \includegraphics[width=\linewidth]{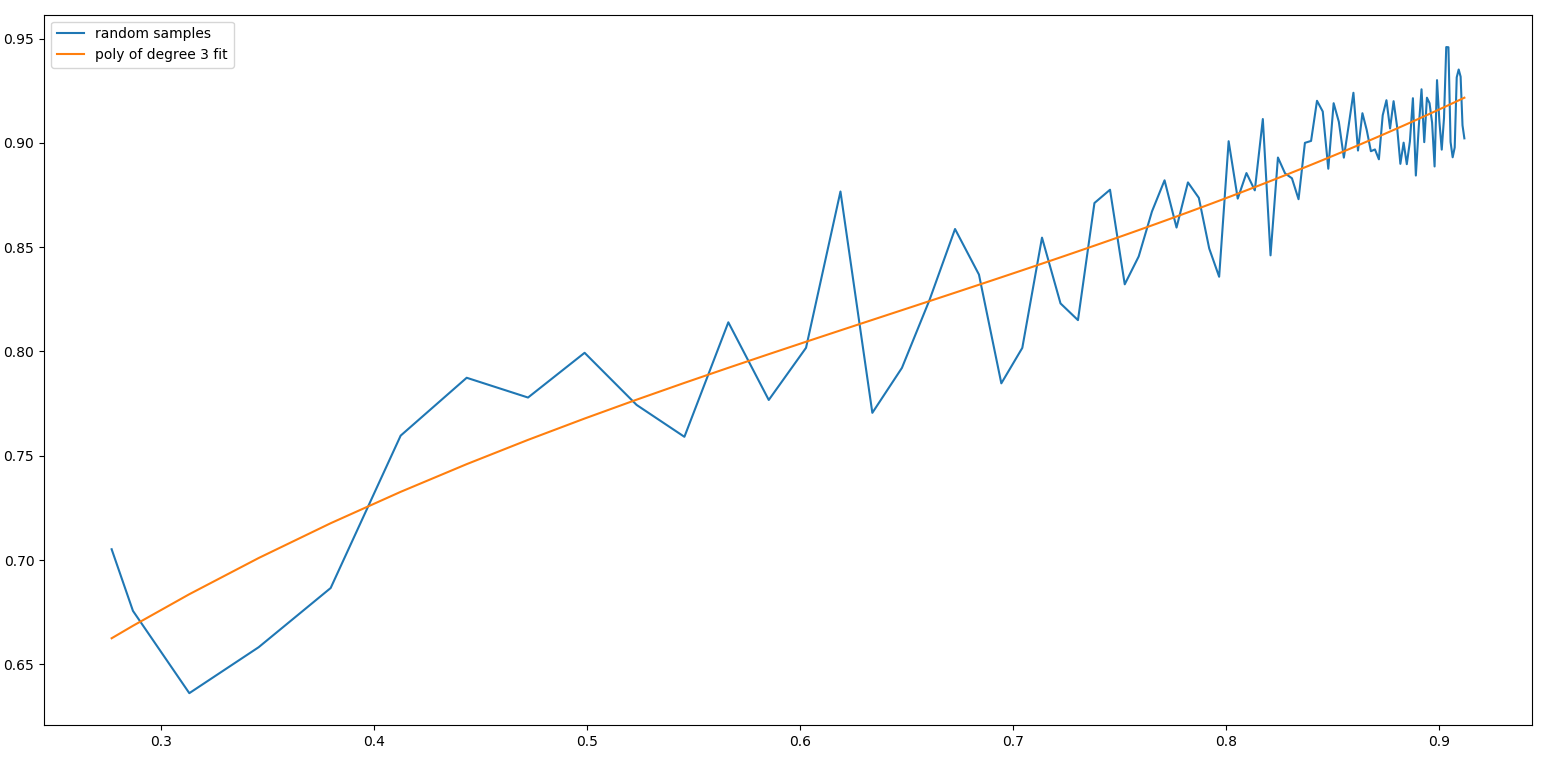}
     \caption{Accuracy v.s Probability of orthogonality, cutoff n=300 for small data set with ~600 unique words}
    \label{figy1}
   \end{minipage}\hfill
   \begin{minipage}{0.45\textwidth}
     \centering
     \includegraphics[width=\linewidth]{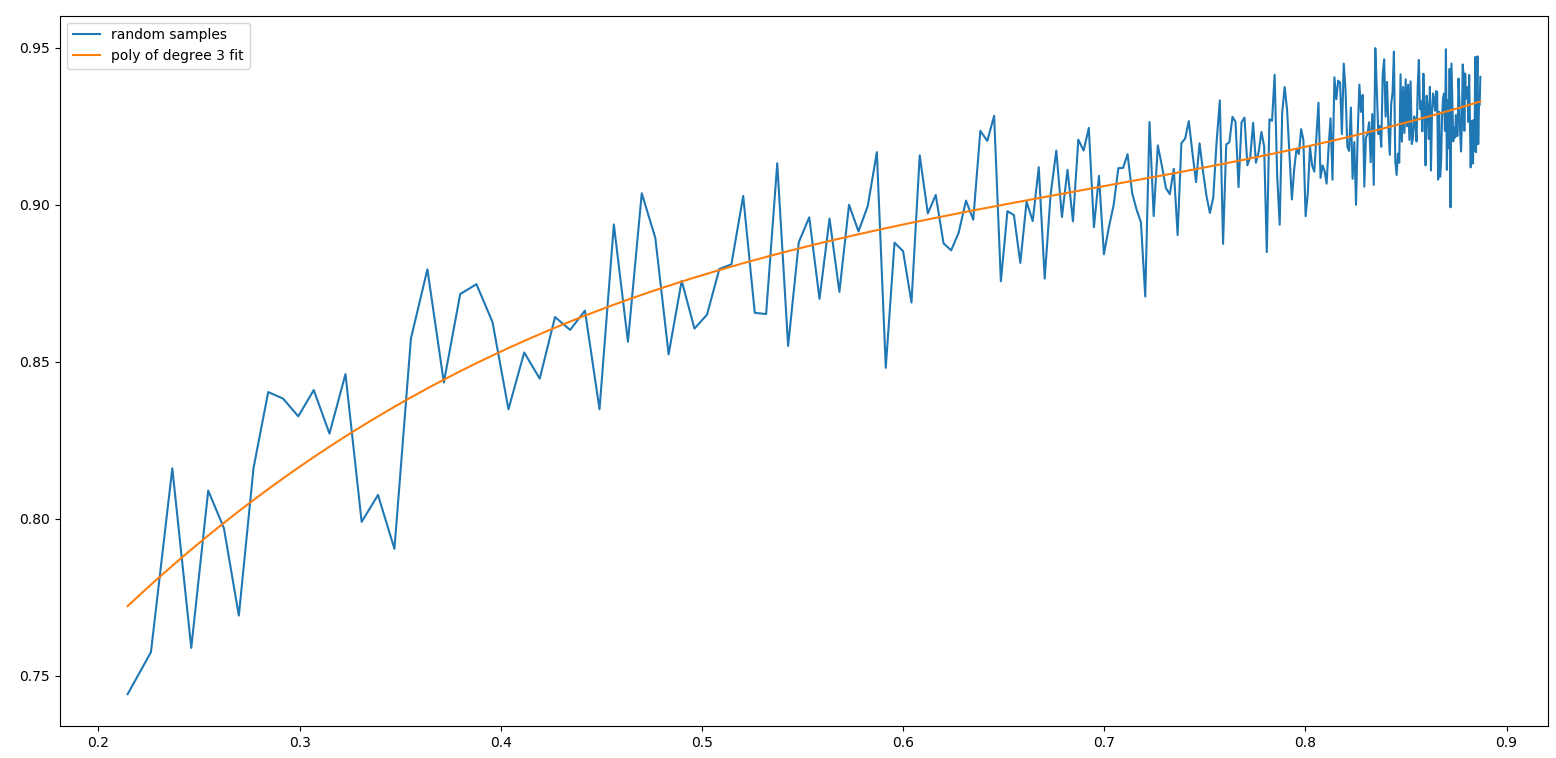}
     \caption{Accuracy v.s Probability of orthogonality for cutoff n=300 for large data set with ~1700 unique words}
    \label{figy2}
   \end{minipage}
\end{figure}
\end{subfigures}

\begin{subfigures}
\begin{figure}[h]
   \begin{minipage}{0.49\textwidth}
     \centering
     \includegraphics[width=\linewidth, height=3.5cm]{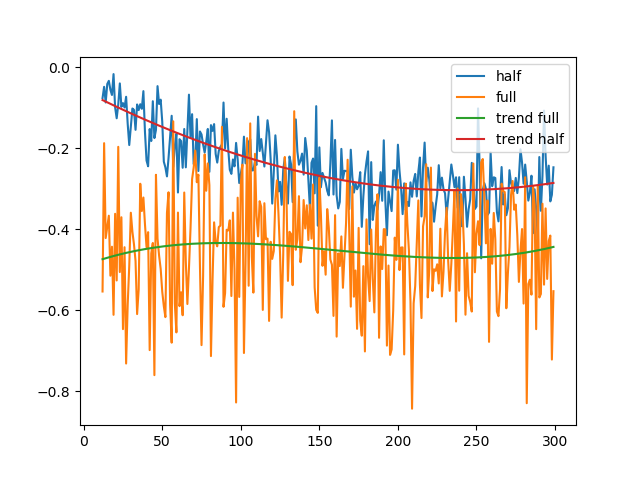}
     \caption{Comparison of Semantic shift between iPhone and X  v.s. dimension of context vectors for representation proposed in \cite{Chatterjee2013EffectON} and our representation for m=5, small data set}
    \label{figz1}
   \end{minipage}\hfill
   \begin{minipage}{0.49\textwidth}
     \centering
     \includegraphics[width=\linewidth, height=3.5cm]{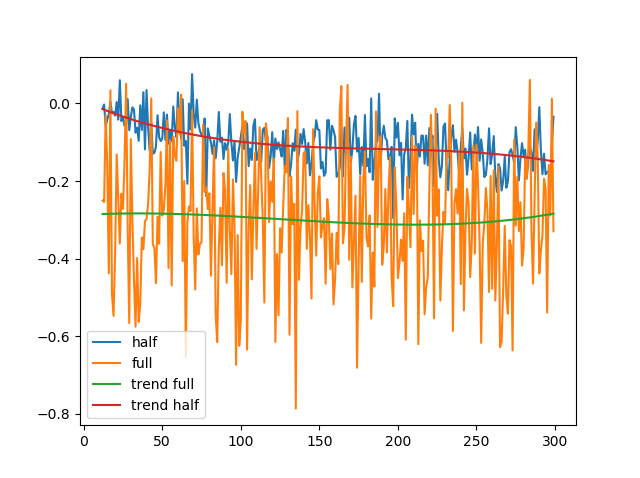}
     \caption{Comparison of Semantic shift between iPhone and X v.s. dimension of context vectors for representation proposed in \cite{Chatterjee2013EffectON} and our representation for m=5, large data set}
    \label{figz2}
   \end{minipage}
\end{figure}

\end{subfigures}

\begin{table}[h]

\label{table4}
 			\centering\caption{Cut-off values of n for RI case along with word capacity}
\begin{adjustbox}{width=\textwidth, height =1cm}

\begin{tabular}{p{0.53in}p{1.83in}p{2.08in}p{1.83in}}
\hline
\multicolumn{1}{|p{0.53in}}{} & 
\multicolumn{1}{|p{1.83in}}{$p>90\%$ } & 
\multicolumn{1}{|p{2.08in}}{$p>95\%$ } & 
\multicolumn{1}{|p{1.83in}|}{$p>97.5\%$ } \\
\hline
\multicolumn{1}{|p{0.53in}}{r=2} & 
\multicolumn{1}{|p{1.83in}}{40  N= 3120} & 
\multicolumn{1}{|p{2.08in}}{80  N= 12640} & 
\multicolumn{1}{|p{1.83in}|}{160  N=50880} \\
\hline
\multicolumn{1}{|p{0.53in}}{r=3} & 
\multicolumn{1}{|p{1.83in}}{87  N= 847960} & 
\multicolumn{1}{|p{2.08in}}{177  N= 7268800} & 
\multicolumn{1}{|p{1.83in}|}{357  N= 60156880} \\
\hline
\multicolumn{1}{|p{0.53in}}{r=4} & 
\multicolumn{1}{|p{1.83in}}{153  N= 351165600} & 
\multicolumn{1}{|p{2.08in}}{314  N= 6357666016} & 
\multicolumn{1}{|p{1.83in}|}{634  N= 106696002016} \\
\hline
\multicolumn{1}{|p{0.53in}}{r=5} & 
\multicolumn{1}{|p{1.83in}}{238  N= 195204469824} & 
\multicolumn{1}{|p{2.08in}}{488  N= 7230043079424} & 
\multicolumn{1}{|p{1.83in}|}{988  N= 248514122298624} \\
\hline
\multicolumn{1}{|p{0.53in}}{r=6 } & 
\multicolumn{1}{|p{1.83in}}{341  N= 133710757852672} & 
\multicolumn{1}{|p{2.08in}}{701  N= 10323765985980160} & 
\multicolumn{1}{|p{1.83in}|}{1000+} \\
\hline

\end{tabular}
\end{adjustbox}
 \end{table}

\section{Conclusion}

In this paper we have shown that embeddings generated through Random Indexing of context vectors can be an efficient substitute for the sparse BoW embeddings used in traditional NLP tasks with highly dynamic datasets. The present work also provide a framweork for analysing the goodness of the embeddings to represent semantic information by analysing the probability of orthogonality. The RI embeddings are shown to be computed faster than BoW embeddings and also provide an added scope of slicing through time in the wake of their constant dimension. 
\par Furthermore, we provide with a novel idea of Probabilistic RI embeddings, which are shown to have larger capacity than both RI and BoW embeddings for a given dimension and can be used for online tasks where the corpus keeps changing. The ammortized time complexity of Probabilistic RI embeddings is marginally better than that of RI embeddings and much better than sparse BoW. They also have better representation capacity and probability of orthogonality. These may be further explored for online tasks. 
\par We further develop an algorithm to track semantic shift in relationship of words from scratch. The event detection task is closely related with tracking semantic shift in word meanings. These embeddings prove to be performing at par with the sparse BoW embeddings in tracking semantic shift for the task at hand and are more efficient in computation as has been shown in the results. These embeddings can also be chosen as a better initialisation point for further learning tasks where sparse BoW is rather expensive in computation and memory if the corpus is large in size. Furthermore, it is not feasible to use pre-trained embeddings or train them in a short time. However, the problem of dynamic detection of change in semantic shift becomes tractable with Probabilstic RI, thus making the proposed method novel.

\bibliographystyle{spmpsci}
\bibliography{ri}
\end{document}